%% file: main_icml2026.tex
\documentclass{article}

\usepackage{microtype}
\usepackage{graphicx}
\usepackage{subcaption}
\usepackage{booktabs} 

\usepackage{algorithm2e}
\usepackage{hyperref}
\usepackage{subcaption}

\newcommand{\R}{\mathbb{R}}
\newcommand{\N}{\mathbb{N}}

\usepackage[accepted]{icml2026}
\usepackage{amsmath}
\usepackage{amssymb}
\usepackage{mathtools}
\usepackage{amsthm}
\usepackage[table]{xcolor}
\usepackage{multirow}
\usepackage{amsmath,amssymb}

\usepackage[capitalize,noabbrev]{cleveref}

\theoremstyle{plain}
\newtheorem{theorem}{Theorem}[section]
\newtheorem{proposition}[theorem]{Proposition}

\theoremstyle{definition}

\theoremstyle{remark}

\usepackage[textsize=tiny]{todonotes}

\icmltitlerunning{Selective Sinkhorn Routing for Improved Sparse Mixture of Experts}

\begin{document}

\twocolumn[
  \icmltitle{Selective Sinkhorn Routing for Improved Sparse Mixture of Experts}

  \icmlsetsymbol{equal}{*}
  \icmlsetsymbol{huu-intern}{\textdagger}
  \icmlsetsymbol{duc-partial}{\textdaggerdbl}

  \begin{icmlauthorlist}
    \icmlauthor{Duc Anh Nguyen}{comp}
    \icmlauthor{Huu Binh Ta}{uva,huu-intern}
    \icmlauthor{Duc Nhuan Le}{comp,duc-partial}
    \icmlauthor{Tan Minh Nguyen}{comp,yyy}
    \icmlauthor{Toan Tran}{comp}
  \end{icmlauthorlist}

  \icmlaffiliation{comp}{Qualcomm AI Research. Qualcomm AI Research is an initiative of Qualcomm Technologies, Inc.}
  \icmlaffiliation{yyy}{Department of Mathematics, National University of Singapore}
  \icmlaffiliation{uva}{University of Virginia, USA}

  \icmlcorrespondingauthor{Duc Anh Nguyen}{nguyenducanh909.bkhn@gmail.com}

  \icmlcorrespondingauthor{Toan Tran}{toatran@qti.qualcomm.com}

  \icmlkeywords{Machine Learning, ICML}

  \vskip 0.3in
]



\printAffiliationsAndNotice{\textsuperscript{\textdagger}Work done at Qualcomm AI Research. \textsuperscript{\textdaggerdbl}Partially done at Ho Chi Minh City University of Science, VNU-HCM, Vietnam.}

\input{Section/Abstract}
\input{Section/Introduction}
\input{Section/Related_work}
\input{Section/Preliminaries}

\input{Section/Method}
\input{Section/Experiments}
\input{Section/Ablation}
\input{Section/Conclusion}

\section*{Impact Statement}
This paper presents work whose goal is to advance the field of Machine
Learning. There are many potential societal consequences of our work, none
which we feel must be specifically highlighted here.




\nocite{langley00}

\bibliography{references}
\bibliographystyle{icml2026}

\newpage
\appendix
\input{Section/Appendix}

\end{document}

%% file: Section/Abstract.tex
\begin{abstract}
Sparse Mixture-of-Experts (SMoE) models are scalable and computationally efficient, enabling large increases in model capacity with limited inference overhead. Existing SMoE methods often depend on auxiliary objectives, such as load-balancing loss and z-loss, or additional trainable components such as noisy gating. While these techniques encourage expert diversity, they can introduce objective misalignment, increase model complexity, or incur substantial training overhead, especially in Sinkhorn-based routing methods. In this paper, we revisit the token-to-expert assignment as an optimal transport problem. We add constraints to ensure balanced expert utilization. We show that even minimal optimal transport-based routing improves SMoE performance without requiring auxiliary balancing losses. Unlike prior approaches, our method derives gating scores directly from the transport map, leading to more balanced and effective token-to-expert assignments. Building on this insight, we introduce Selective Sinkhorn Routing (SSR), a lightweight routing mechanism that replaces complex auxiliary losses with efficient Sinkhorn-based routing while preserving flexible expert selection. Experiments on language modeling and image classification show that SSR improves training efficiency, accuracy, and robustness to input corruption.
\end{abstract}

%% file: Section/Introduction.tex
\section{Introduction}

Foundation models have rapidly advanced across language \cite{Vaswmani2017_attention, Brown2020_GPT-3, Colin2020_T5}, vision \cite{Alexey2020_ViT, liu2021_Swin, Riquelme2021_vmoe}, and multimodal learning \cite{Bin2024_MoELLaVa, Rasheed2024_GLaMM}. A common way to improve their performance is to scale model capacity, but dense scaling substantially increases computational cost and inference latency. Sparse Mixture-of-Experts (SMoE) models \cite{shazeer2017_outrageous, lepikhin2020_gshard, fedus2022_switch} provide an efficient alternative by activating only a small subset of specialized experts for each token, thereby increasing capacity while limiting computation.

Conventional SMoE models use Softmax-based routing, assigning each token to the top-$k$ experts based on gating scores. This can cause routing collapse, where only a small subset of experts is frequently selected while others remain underutilized \cite{chi2022_representation-collapse}. Prior work mitigates this issue with auxiliary load-balancing losses \cite{fedus2022_switch, lepikhin2020_gshard}, but such losses may introduce training instability \cite{zoph2022_stmoe}. Alternatively, OT-based routing methods use Sinkhorn algorithms to obtain balanced expert assignments without auxiliary losses \cite{liu2022_OT_MoE}, though they may reduce routing flexibility because the gating matrix is not directly optimized through gradient-based learning \cite{liu2024_router}.

In contrast to prior Sinkhorn-based approaches \citep{kool2021_unbiased, clark2022_scaling-law, liu2024_router}, which use the transport map only to select the top-$k$ experts, we also use its values to assign weights to each selected expert. We derive routing weights from transport-map values rather than from the conventional gating-score matrix, leveraging Sinkhorn’s built-in expert balancing in SMoE. This design improves expert balancing, supported by both theory and empirical results. We introduce Selective Sinkhorn Routing (SSR), a lightweight routing mechanism that replaces auxiliary losses with minimal Sinkhorn-based optimization. Applying SSR to only 0.1\%--1\% of training steps per epoch yields faster convergence, higher accuracy, and greater robustness to input corruption. We also show that although Sinkhorn routing and the noise-addition trick improve training performance, they degrade inference performance.
This work makes the following principal contributions:
\vspace{-0.5em}
\begin{itemize}
    \setlength{\itemsep}{-0.4em}
    \setlength{\topsep}{-0.4em}
    \item We propose a routing framework for SMoE models that integrates entropy-regularized optimal transport with stochastic noise injection to promote balanced expert utilization and improve training stability.
    \item We provide theoretical results showing that Sinkhorn-based routing and noise injection aid training by encouraging exploration and expert balancing, but should be disabled at inference to ensure consistent, deterministic routing.
    \vspace{-0.3em}
    \item We conduct extensive evaluations on language modeling and vision tasks, demonstrating that the proposed approach outperforms existing methods.
\end{itemize}

%% file: Section/Related_work.tex
\section{Related Work}
\label{sec:ablation}

\paragraph{Sparse Mixture of Experts}
Sparse Mixture-of-Experts \citep{shazeer2017_outrageous, du2022_glam, fedus2022_switch} has become a core backbone in deep learning, increasing model capacity while preserving computational efficiency. Compared to densely activated models, SMoE can improve performance across tasks without excessive compute \citep{lepikhin2020_gshard, zhou2022_expert-choice}. By activating only a subset of experts per input, SMoE substantially increases parameter count without increasing FLOPs per example. This sparsity supports large-scale training that remain cost-effective at inference, making SMoE attractive for applications such as language modeling and machine translation.
\vspace{-2em}
\paragraph{Routing in SMoE}
In SMoE architectures, the router assigns tokens to experts. A common approach uses a gating network with a softmax to produce a distribution over experts for each token \citep{shazeer2017_outrageous}. To balance token load, load-balancing losses \citep{shazeer2017_outrageous, lepikhin2020_gshard, fedus2022_switch} or injected noise \citep{shazeer2017_outrageous} are often used. Despite improving utilization, they introduce extra hyperparameters and training complexity. Sinkhorn-based routing instead optimizes token--expert assignments via the Sinkhorn algorithm \citep{clark2022_scaling-law, cai2024_survey}. ST-MoE \citep{zoph2022_stmoe} adds a $z$-loss penalizing overly large gating logits for numerical stability.

%% file: Section/Preliminaries.tex
\vspace{-1.5em}
\section{Preliminaries}
\label{sec:preliminaries}


This section provides the foundation on SMoE models, followed by the background of the Sinkhorn-Knopp Algorithm.
\subsection{Sparse Mixture of Experts}
\label{subsec:SMoE}


An SMoE model comprises multiple MoE blocks, each containing a set of experts. Within each block, experts process different aspects of the input, and their outputs are combined to form the block output. Let $\mathbf{X} \in \mathbb{R}^{m \times d}$ denote the matrix of $m$ input token embeddings, and let $\mathbf{W}_g \in \mathbb{R}^{n \times d}$ denote the gating weight matrix for $n$ experts. The gating score matrix $\mathbf{S} \in \mathbb{R}^{m \times n}$ is computed as
$\mathbf{S} = \mathbf{X}\mathbf{W}_g^\top$.


Each entry $s_{i,j}$ in $\mathbf{S}$ is a compatibility score between token $i$ and expert $j$, indicating how suitable expert $j$ is for processing token $i$. Larger scores imply a stronger preference for routing token $i$ to expert $j$. For each token $i \in \{1,\dots,m\}$, we select the top-$k$ experts with the highest scores, denoted by the index set $\mathcal{T}_i \subseteq \{1,\dots,n\}$ with $|\mathcal{T}_i|=k$. The routing weight from token $i$ to expert $j$ is defined as
\[
w_{i,j}=
\begin{cases}
\displaystyle \frac{\exp(s_{i,j})}{\sum_{j' \in \mathcal{T}_i}\exp(s_{i,j'})}, & \text{if } j \in \mathcal{T}_i,\\[6pt]
0, & \text{otherwise.}
\end{cases}
\]
Each expert $j$ is represented by a feedforward network $f_j:\mathbb{R}^d \rightarrow \mathbb{R}^d$. Token $i$'s output is aggregated as
$\mathbf{y}_i = \sum_{j \in \mathcal{T}_i} w_{i,j}\, f_j(\mathbf{x}_i)$.

\subsection{Sinkhorn-Knopp Algorithm}
Given two probability vectors $r \in \R^m, s \in \R^n$ and a cost matrix $\mathbf{C} \in \mathbb{R}^{m \times n}$, the entropy-regularized optimal transport (OT) problem (with entropic regularization parameter $\xi > 0$) seeks a transport plan $\boldsymbol{\hat{\Pi}} \in \mathbb{R}^{m \times n}$ that minimizes:
\begin{equation}
 \boldsymbol{\hat{\Pi}} = \underset{\boldsymbol{\Pi} \in \mathbb{R}^{m \times n}}{\arg\min} \langle \boldsymbol{\Pi}, \mathbf{C} \rangle + \xi \langle \boldsymbol{\Pi}, \log \boldsymbol{\Pi} \rangle, 
\end{equation}
\vspace{-0.8em}
\begin{equation}
\text{subject to: } \boldsymbol{\Pi} > 0, \ \ \boldsymbol{\Pi} \mathbf{1}_n = r, \ \ \boldsymbol{\Pi}^\top \mathbf{1}_m = s.
\end{equation}

This can be solved efficiently by the Sinkhorn-Knopp algorithm \cite{Cuturi2013_Sinkhorn} (see~\cref{alg:sinkhorn_alg}), which alternates between scaling rows and columns of a kernel matrix $\mathbf{K} = \exp(-\mathbf{C} / \xi)$ to match the marginals $r$ and $s$.









\begin{algorithm}[h]
\small
\captionsetup{font=small}
\caption{Iterative Sinkhorn-Knopp Algorithm}
\label{alg:sinkhorn_alg}
\KwIn{Cost matrix $\mathbf{C} \in \R^{m \times n}$, $r \in \R^m ,s \in \R^n$, $\xi, \delta > 0$, number of iterations $\eta \in \N^{+}$.}
\KwOut{Optimal matrix $\boldsymbol{\Pi}^* \in \R^{m \times n}$.}
$\mathbf{u}^{(0)} = 1_m, \mathbf{K} = \exp(-\mathbf{C}/\xi)$.

\For{$each \; k = 1\; to\; \eta$}{
$\mathbf{v}^{(k)} = s \,\oslash \mathbf{K}^\top\mathbf{u}^{(k-1)}$; \hfill // Element-wise division

$\mathbf{u}^{(k)} = r \, \oslash \mathbf{K}\mathbf{v}^{(k)}$; \hfill // Element-wise division

\uIf{$\|\boldsymbol{\Pi} \mathbf{1}_n - r\| < \delta$ and $\|\boldsymbol{\Pi}^\top \mathbf{1}_m - s\| < \delta$}{
break;
}
}

$\boldsymbol{\hat{\Pi}} = \operatorname{diag}(\mathbf{u}^{(k)}) \, \mathbf{K} \, \operatorname{diag}(\mathbf{v}^{(k)})$.
\end{algorithm}

%% file: Section/Method.tex
\vspace{-1em}
\section{Methodology}
\label{Method}
In this section, we propose a novel method to balance token allocation in SMoE models via an optimal-transport-based token-to-expert assignment mechanism. 
A key challenge is that it requires reasonably good gating scores to produce meaningful transport maps. To address this, we concurrently employ standard Softmax gating during training to update the gating weight matrices. This joint training mechanism lets transport-based routing leverage updated gating scores while preserving its assignment process. Unless otherwise stated, all notation follows \cref{sec:preliminaries}. Proofs are provided in the Appendix.


\subsection{Token-to-expert Assignment as an Entropy-regularized Optimal Transport Problem}
\label{sec:method_token_to_expert_assignment}



To achieve effective load balancing, we formulate token-to-expert assignment as an entropy-regularized maximum-cost optimal transport (OT) problem that assigns tokens to experts by maximizing overall compatibility. We propose two approaches to construct the transport cost matrix $\mathbf{C} \in \mathbb{R}^{m \times n}$ from the gating score matrix $\mathbf{S}$: (1) \textbf{Raw gating scores (linear cost)}: $\mathbf{C}=\mathbf{S}$; (2) \textbf{Normalized scores (softmax cost)}: apply a row-wise softmax to $\mathbf{S}$, i.e., $\mathbf{C}_{i,:}=\mathrm{softmax}(\mathbf{S}_{i,:})$. The softmax cost prevents kernel entries from becoming excessively large, since unbounded values can cause exponential growth and numerical overflow during Sinkhorn updates.

Building on this motivation, we further impose balancing constraints across experts, leading to an entropy-regularized maximum-cost OT problem. This problem can be efficiently solved using the Sinkhorn algorithm, formulated as:
\vspace{-0.5em}
\begin{align}
\label{OT_lb}
\boldsymbol{\hat{\Pi}} 
= \underset{\boldsymbol{\Pi} \in \mathbb{R}^{m \times n}}{\arg\max}
\left\{
\langle \boldsymbol{\Pi}, \mathbf{C} \rangle
- \xi \langle \boldsymbol{\Pi}, \log \boldsymbol{\Pi} \rangle
\right\},
\end{align}
\vspace{-0.6em}
subject to
\begin{equation}
\text{(C1) } \boldsymbol{\Pi} > 0, \quad
\text{(C2) } \boldsymbol{\Pi} \mathbf{1}_n = \mathbf{1}_m, \quad
\text{(C3) } \boldsymbol{\Pi}^\top \mathbf{1}_m = \tfrac{m}{n} \mathbf{1}_n.
\end{equation}
We clarify those constraints as follows. (C1): All routing probabilities must be positive to form well-defined entropy term. (C2): Each token routes its entire mass to experts; each row of $\boldsymbol{\Pi}$ sums to $1$, so all token information is preserved. (C3): Each expert receives the same expected total load; each column of $\boldsymbol{\Pi}$ sums to $m/n$.
Unlike existing Sinkhorn Token Choice routers~\cite{kool2021_unbiased, clark2022_scaling-law}, which use the transport plan only for expert selection, we directly use $\boldsymbol{\hat{\Pi}}$ to compute routing weights. This approach is theoretically supported by \cref{prop:optimal_transport_routing}. Specifically, for each token $i$, we select the top-$k$ experts with the largest entries in the $i$-th row of $\boldsymbol{\hat{\Pi}}$. 
Let token $i$ be assigned to experts $E_{i_1}, E_{i_2}, \ldots, E_{i_k}$, corresponding to the $k$ largest entries in the $i$-th row of $\boldsymbol{\hat{\Pi}}$. The compatibility score between token $i$ and expert $E_{i_r}$ is given by $\boldsymbol{\hat{\Pi}}_{i, i_r}$, where $r \in \{1, 2, \ldots, k\}$. 

\begin{proposition}
\label{prop:optimal_transport_routing}
Let $\boldsymbol{\hat{\Pi}} \in \mathbb{R}^{m \times n}$ be the solution to the entropy-regularized optimal transport problem in Eq.~\ref{OT_lb}. For each token $i \in \{1, \ldots, m\}$, suppose we are allowed to assign it to at most $k$ experts, with routing weights ${\alpha}_i \in \mathbb{R}^n$ satisfying:
{
\setlength{\abovedisplayskip}{0.5pt}
\setlength{\belowdisplayskip}{0.5pt}
\begin{equation}
    \alpha_{i, j} \geq 0, \quad \text{supp}(\alpha_i) \leq k, \quad \sum_{j=1}^n \alpha_{i, j} = 1.
\end{equation}
}
We select the top-$k$ experts $E_{i_1}, E_{i_2}, \ldots, E_{i_k}$ with the highest transport scores $\boldsymbol{\hat{\Pi}}_{i, i_1}$, $\boldsymbol{\hat{\Pi}}_{i, i_2}, \ldots, \boldsymbol{\hat{\Pi}}_{i, i_k}$, and set the weights as:
$\alpha_{i, i_r} = \frac{\boldsymbol{\hat{\Pi}}_{i, i_r}}{\sum_{j=1}^{k} \boldsymbol{\hat{\Pi}}_{i, i_j}} \quad \text{for } r = 1, \ldots, k$.
Then, $\alpha_i$ is the optimal solution of the optimization problem $\min\limits_{{\alpha}_i} \mathrm{KL}(\alpha_i \,\|\, \boldsymbol{\hat{\Pi}}_{i})$.

\end{proposition}

\subsection{Selective Sinkhorn Routing for Sparse Mixture of Experts}



A key limitation of Sinkhorn-based token-to-expert assignment is that it does not update the gating weight matrix $\mathbf{W}_g$. While this preserves token information during routing, $\mathbf{W}_g$ is decoupled from the computational graph because it is not included in the OT objective. As a result, the gating score matrix $\mathbf{S}$ is not directly optimized for token--expert compatibility, weakening the semantic meaning of the transport map and its role as a compatibility-aware supervision signal. Thus, the OT formulation is meaningful only when $\mathbf{S}$ reliably captures token--expert compatibility.


\begin{proposition}[Load Balancing Is Valid During Training but Not Inference]
\label{prop:sinkhorn_for_training}

Let $\mathcal{X}$ be the input space and $p(x)$ the data distribution. Let $g: \mathcal{X} \rightarrow \mathbb{R}^n$ be a gating function producing expert scores $\mathbf{S}(x) \in \mathbb{R}^n$, and let $f: \mathbb{R}^{m \times n} \rightarrow \mathbb{R}^{m \times n}$ be a routing function based on entropy-regularized optimal transport. For a batch $\{x_1, \dots, x_m\}$, the transport plan $\boldsymbol{\Pi} \in \mathbb{R}^{m \times n}$ satisfies:
\begin{equation}
\text{(C1) } \boldsymbol{\Pi} > 0 \nonumber, \quad \text{(C2) } \boldsymbol{\Pi} \mathbf{1}_n = \mathbf{1}_m \nonumber, \quad \text{(C3) } \boldsymbol{\Pi}^\top \mathbf{1}_m = \tfrac{m}{n} \mathbf{1}_n.
\end{equation}
Then:


\textbf{(1) (Training)} When the batch $\{x_1, \dots, x_m\}$ consists of i.i.d.\ samples from $p(x)$, the average routing approximates the expected routing $\mathbb{E}_{x \sim p(x)}[\pi(x)]$. Constraint (C3) encourages uniform expert usage under $p(x)$.

\textbf{(2) (Inference)} When a single input $x$ is processed, there is no meaningful approximation to $p(x)$, so enforcing (C3) forces uniform routing regardless of the actual score $\mathbf{S}(x)$, which leads to distorted or suboptimal expert assignment.

\end{proposition}

To address this issue, we propose \emph{Selective Sinkhorn Routing} (SSR). During training, each MoE block uses Sinkhorn routing with probability $p \in [0,1]$ and Softmax gating otherwise. This hybrid strategy allows Softmax gating to train $\mathbf{W}_g$, enabling $\mathbf{S}$ to learn meaningful token--expert compatibilities, while Sinkhorn routing promotes balanced expert utilization and improves training stability. We denote the Linear-cost and Softmax-cost variants as SSR-L and SSR-S, respectively. The overall procedure is illustrated in \cref{fig:methodology}.

\begin{algorithm}[h]
\small
\captionsetup{font=small}
\caption{Selective Sinkhorn Routing (SSR) in an MoE block during training}
\label{alg:ssr}
\KwIn{Gating scores $\mathbf{S} \in \mathbb{R}^{m \times n}$, Sinkhorn probability $p$, top-$k$ experts $k$.}
\KwOut{Routing weights $\boldsymbol{\alpha} \in \mathbb{R}^{m \times n}$.}

Set $\alpha_{i,j} \leftarrow 0$ for all $i = 1,\ldots,m$ and $j = 1,\ldots,n$\;

Draw $\tau \sim \mathcal{U}(0, 1)$\;

\uIf{$\tau < p$}{
    Compute transport plan $\boldsymbol{\hat{\Pi}}$ from $\mathbf{S}$ using Sinkhorn algorithm\;

    \For{each token $i = 1$ to $m$}{
        Select $k$ experts $E_{i_1}, \ldots, E_{i_k}$ with highest
        $\left\{ \boldsymbol{\hat{\Pi}}_{i,j} \right\}_{j=1}^n$\;

        \For{$r = 1$ to $k$}{
            $\alpha_{i,i_r} \leftarrow
            \dfrac{\boldsymbol{\hat{\Pi}}_{i,i_r}}
            {\sum_{j=1}^k \boldsymbol{\hat{\Pi}}_{i,i_j}}$\;
        }
    }
}
\Else{
    \For{each token $i = 1$ to $m$}{
        Select $k$ experts $E_{i_1}, \ldots, E_{i_k}$ with highest
        $\left\{ \mathbf{S}_{i,j} \right\}_{j=1}^n$\;

        \For{$r = 1$ to $k$}{
            $\alpha_{i,i_r} \leftarrow
            \dfrac{e^{\mathbf{S}_{i,i_r}}}
            {\sum_{j=1}^k e^{\mathbf{S}_{i,i_j}}}$\;
        }
    }
}
\end{algorithm}

However, the Softmax-based routing may still suffer from expert collapse, in which only a small subset of experts is consistently selected due to high gating scores \cite{cai2024_survey}. To address this issue, during training, we can add Gaussian noise to the cost matrix (SSR with noise) to encourage exploration and prevent expert underutilization. In particular, the noisy cost matrix is computed as:
\begin{equation}
    \tilde{\mathbf{C}} = \mathbf{C} + \alpha_{\text{noise}} \cdot \boldsymbol{\epsilon}, \quad \boldsymbol{\epsilon} \in \mathbb{R}^{m \times n}, \quad \boldsymbol{\epsilon}_{i,j} \sim \mathcal{N}(0, \sigma^2).
\end{equation}

\cref{alg:ssr} desribes the training-time behavior of an SMoE block under SSR. Compared with auxiliary-loss methods, SSR avoids objective misalignment by removing additional balancing losses and instead encouraging balanced expert utilization through entropy-regularized optimal transport. Compared with existing Sinkhorn-based routing, SSR preserves the OT interpretation while reducing overhead via sparse Sinkhorn updates. Unlike trainable noise-injection methods, SSR w/ noise introduces no additional parameters, maintaining simplicity without sacrificing performance.

\begin{proposition}[Noise Ensures Every Expert Has Nonzero Selection Probability]
\label{prop:noise_positive_selection}
For each token, let $g = (g_1, \dots, g_n) \in \mathbb{R}^n$ be the cost to $n$ experts, and let the perturbed costs be: $
\tilde{g}_i = g_i + \alpha_{\text{noise}} \epsilon_i, \quad \epsilon_i \sim \mathcal{N}(0, \sigma^2)$. Let $P_i = \mathbb{P}(\tilde{g}_i > \tilde{g}_j \text{ for all } j \ne i)$ denote the probability that expert $i$ is selected as the top-1 expert after noise is added.

Then:
$$
P_i = \prod_{j \ne i} \Phi\left( \frac{g_i - g_j}{\sqrt{2} \sigma \alpha_{\text{noise}}} \right),
$$
where \( \Phi(z) = \int_{-\infty}^z \frac{1}{\sqrt{2\pi}} e^{-t^2/2} dt \) is the CDF of the standard normal distribution.

\end{proposition}

\begin{figure}
    \centering
    \includegraphics[width=0.78\linewidth]{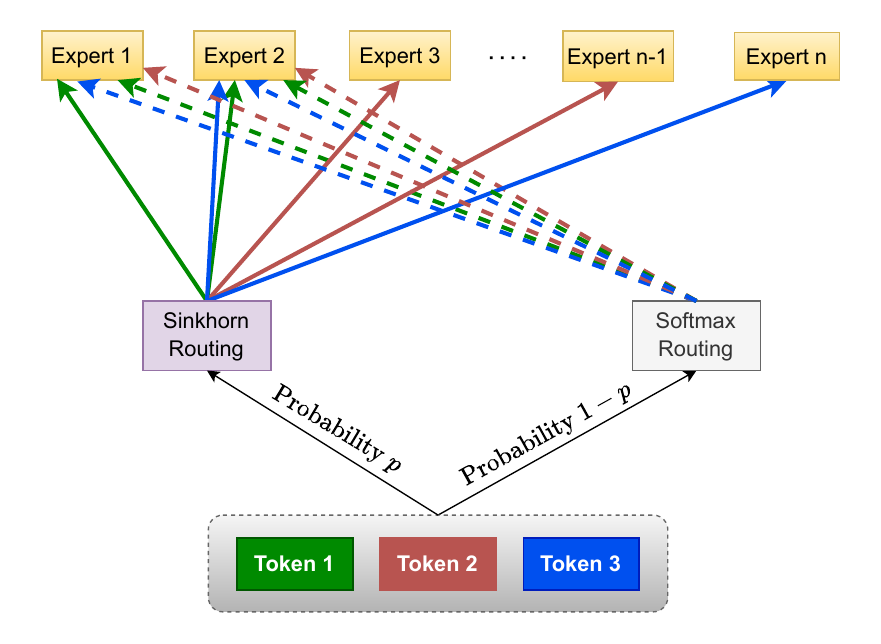}
    \caption{\textbf{Training of SSR.} At each layer during the forward pass, a fixed probability $p \in [0,1]$ is introduced. We randomly choose Sinkhorn-based routing with probability $p$ and Softmax-based routing with probability $1-p$.}
    \label{fig:methodology}
\vspace{-0.2cm}    
\end{figure}

By \cref{prop:sinkhorn_for_training,prop:noise_positive_selection}, both Sinkhorn routing and noise injection improve training by promoting expert load balancing. In SSR-S, \cref{prop:noise_positive_selection} gives \(P_i \ge \Phi^n\!\left(-\frac{1}{\sqrt{2}\,\sigma\,\alpha_{\text{noise}}}\right)\), which depends only on the noise hyperparameter. 
At inference, we disable both mechanisms and use deterministic Softmax routing to avoid stochastic or batch-dependent behavior.

%% file: Section/Experiments.tex
\vspace{-1em}
\section{Experiments}
\vspace{-0.5em}
This section presents the experimental setup, compares SSR with current state-of-the-art balancing baselines, and provides ablation studies. Unless otherwise specified, all experiments are conducted under the conventional SMoE setting. We additionally report Momentum-enhanced results on WikiText-103 and selected ablations to evaluate the compatibility of SSR with Momentum dynamics.
\vspace{-1em}
\subsection{Settings}
\paragraph{Datasets}
We evaluate SSR on both text and vision tasks using six datasets: WikiText-103 \cite{merity2017_wikitext} and enwik8 \cite{mahoney2011large} for language modeling, and ImageNet-1K \cite{Deng2009_imagenet}, ImageNet-A, ImageNet-O \cite{hendrycks2021natural}, and ImageNet-R \cite{hendrycks2021many} for vision robustness evaluation.
\vspace{-1em}
\paragraph{Baselines} 
For text tasks, we use the Switch Transformer architecture \cite{fedus2022_switch}; for vision tasks, we use Swin Transformer \cite{liu2021_Swin}. We compare against Vanilla SMoE (SMoE w/o balancing), SMoE with load balancing loss (SMoE w/ lb loss) \cite{lepikhin2020_gshard}, SMoE with (trainable) noise injection and load balancing loss (SMoE w/ noise) \cite{shazeer2017_outrageous}, SMoE with load balancing loss and z-loss (SMoE w/ z-loss) \cite{zoph2022_stmoe}, and Sinkhorn-based SMoE \cite{liu2024_router}.


\begin{table}[t]
\vspace{-0.5em}
\centering
\caption{Perplexity (PPL) of SSR variants and baseline models on clean and attacked WikiText-103 under conventional and momentum settings. The $\Delta$ column reports the test PPL difference relative to Vanilla SMoE within the corresponding setting. Time overhead is measured relative to Vanilla SMoE in each setting.}
\label{tab:wikitext-103-unified}
\renewcommand\tabcolsep{4pt}
\resizebox{0.95\columnwidth}{!}{%
\begin{tabular}{lcccccc}
\toprule
\multirow{3}{*}{Model/Metric} 
& \multicolumn{3}{c}{Clean WikiText-103} 
& \multicolumn{2}{c}{Attacked WikiText-103}
& \multicolumn{1}{c}{Efficiency} \\
\cmidrule(lr){2-4} \cmidrule(lr){5-6} \cmidrule(lr){7-7}
& Valid PPL $\downarrow$ 
& Test PPL $\downarrow$ 
& \shortstack{$\Delta$ Test PPL \\ vs. Baseline}
& Valid PPL $\downarrow$ 
& Test PPL $\downarrow$
& \shortstack{Training Time \\ Overhead $\downarrow$} \\
\midrule

\multicolumn{7}{l}{\textbf{Conventional setting}} \\
\midrule

Vanilla SMoE 
& 33.760
& 35.550
& -- 
& 42.240 
& 44.190 
& 0.00\% \\

SMoE w/ lb loss 
& 33.248
& 34.952 
& \textcolor{green!70!black}{-0.598} 
& 41.763 
& 43.758 
& 2.04 \% \\

SMoE w/ z-loss 
& 33.242 
& 35.091
& \textcolor{green!70!black}{-0.459} 
& 42.298 
& 44.406 
& 2.12\% \\

SMoE w/ noise 
& 33.196 
& 35.072
& \textcolor{green!70!black}{-0.478}
& 41.432
& 43.692 
& 2.50\% \\

Sinkhorn-based SMoE 
& 33.301
& 35.316
& \textcolor{green!70!black}{-0.234} 
& 42.056 
& 44.351 
& 72.47\% \\

\rowcolor{gray!20}
SSR-L (Ours) 
& 32.809 
& 34.573 
& \textcolor{green!70!black}{-0.977} 
& 41.230 
& 43.318 
& 0.33\% \\

\rowcolor{gray!20}
SSR-S (Ours) 
& \textbf{32.610} 
& 34.744 
& \textcolor{green!70!black}{-0.806} 
& \textbf{41.012} 
& 43.283 
& 0.37\% \\

\rowcolor{gray!20}
SSR-L w/ noise (Ours) 
& \underline{32.767} 
& \textbf{34.367} 
& \textcolor{green!70!black}{-1.183} 
& 41.164 
& \underline{43.159} 
& 0.62\% \\

\rowcolor{gray!20}
SSR-S w/ noise (Ours) 
& 32.881 
& \underline{34.557}
& \textcolor{green!70!black}{-0.993} 
& \underline{41.037} 
& \textbf{42.941} 
& 0.65\% \\

\midrule
\multicolumn{7}{l}{\textbf{Momentum setting}} \\
\midrule

Vanilla SMoE
& 31.861 
& 33.712 
& -- 
& 39.721 
& 41.756 
& 0.00\% \\

SMoE w/ lb loss 
& 32.561 
& 34.490 
& \textcolor{red!70!black}{+0.778} 
& 40.811 
& 42.998 
& 2.49\% \\

SMoE w/ z-loss 
& 32.542 
& 34.235 
& \textcolor{red!70!black}{+0.523} 
& 40.665 
& 42.614 
& 4.44\% \\

SMoE w/ noise
& 32.490 
& 34.077 
& \textcolor{red!70!black}{+0.365} 
& 40.299 
& 42.204 
& 6.75\% \\

Sinkhorn-based SMoE 
& 32.576 
& 34.305 
& \textcolor{red!70!black}{+0.593} 
& 40.810 
& 42.803 
& 67.52\% \\

\rowcolor{gray!20}
SSR-L (Ours) 
& \underline{31.834} 
& \underline{33.305} 
& \textcolor{green!70!black}{-0.407} 
& \underline{39.690} 
& \textbf{41.351} 
& 0.36\% \\

\rowcolor{gray!20}
SSR-S (Ours) 
& 31.888 
& 33.559 
& \textcolor{green!70!black}{-0.153} 
& 40.259 
& 42.104 
& 0.61\% \\

\rowcolor{gray!20}
SSR-L w/ noise (Ours) 
& 31.911 
& \textbf{33.223} 
& \textcolor{green!70!black}{-0.489} 
& 39.847 
& \underline{41.627} 
& 1.58\% \\

\rowcolor{gray!20}
SSR-S w/ noise (Ours) 
& \textbf{31.820} 
& 33.338 
& \textcolor{green!70!black}{-0.374} 
& \textbf{39.677} 
& 41.639 
& 1.76\% \\

\bottomrule
\end{tabular}
}
\end{table}

\begin{table}[t]
\vspace{-0.1in}

\centering
\caption{Top-1 and Top-5 accuracy on ImageNet-1K and robustness benchmarks: Top-1 accuracy on ImageNet-A (IM-A) and ImageNet-R (IM-R), and AUPR on ImageNet-O (IM-O).}
\label{table:vision_main}
\renewcommand\tabcolsep{6.5pt}
\resizebox{0.45\textwidth}{!}{
\begin{tabular}{lcccccc}
\toprule
\multirow{3}{*}{Model} 
& \multicolumn{3}{c}{ImageNet-1K} 
& \multicolumn{3}{c}{Robustness Benchmarks} \\
\cmidrule(lr){2-4} \cmidrule(lr){5-7}
& Top-1 $\uparrow$ 
& \shortstack{$\Delta$ Top-1 \\ vs. Vanilla} 
& Top-5 $\uparrow$ 
& IM-A $\uparrow$ 
& IM-O $\uparrow$ 
& IM-R $\uparrow$ \\
\midrule
Vanilla SMoE    
& 75.052 & -- & 92.302 & 6.852 & 50.690 & 30.713 \\

SMoE w/ noise            
& 75.148 & \textcolor{green!70!black}{+0.096} & 92.356 & 7.000 & \underline{50.730} & 30.657 \\

Swin-MoE                 
& 75.322 & \textcolor{green!70!black}{+0.270} & \underline{92.578} & \underline{7.093} & 50.460 & \underline{31.743} \\

\rowcolor{gray!20}
SSR-L (Ours)             
& \underline{75.402} & \textcolor{green!70!black}{+0.350} & 92.528 & 6.600 & \textbf{51.040} & 30.863 \\

\rowcolor{gray!20}
SSR-L w/ noise (Ours)    
& \textbf{77.420} & \textcolor{green!70!black}{+2.368} & \textbf{93.566} & \textbf{9.760} & 50.530 & \textbf{33.903} \\

\bottomrule
\end{tabular}
\vspace{-0.1in}
}
\end{table}

\vspace{-1em}
\paragraph{Evaluation Metrics and Analysis}
For text domain, we report token-level perplexity (PPL) on WikiText-103 \& byte-level bits-per-character (BPC) on Enwik-8, using zero-shot evaluation on adversarial inputs. For images, models trained on ImageNet-1K are evaluated on a clean validation set and on adversarial, OOD, and real-world filtered sets, using Top-1/Top-5 accuracy and AUPR, see Appendix for details.


\subsection{Main results}

We compare our methods with several baselines across multiple benchmarks. In each evaluation, the \textbf{best result is in bold}, and \underline{the second-best is underlined}. In the main tables, the $\Delta$ column reports the gap relative to Vanilla SMoE:
$
\Delta(\text{method}) = \text{Performance(method)} - \text{Performance(Vanilla SMoE)}.
$
A negative $\Delta$ indicates improvement for PPL and BPC, while a positive $\Delta$ indicates improvement for vision tasks (Top-1 accuracy). Thus, the sign of $\Delta$ depends on the metric: lower is better for PPL/BPC, whereas higher is better for Top-1 accuracy. The symbol ``--'' denotes no change (i.e., Vanilla SMoE).

We first evaluate SSR on WikiText-103 under both the conventional and Momentum settings \cite{teo2024_momentumsmoe}. In the conventional setting, the backbone follows an interleaved Attention--SMoE architecture with separate residual connections for each block. As shown in \cref{tab:wikitext-103-unified}, SSR consistently outperforms Vanilla SMoE and existing balancing methods. In particular, SSR reduces test perplexity by \textcolor{green!70!black}{1.183 PPL} relative to Vanilla SMoE, achieving \textbf{more than twice} the improvement of prior balancing approaches, while incurring \textbf{substantially lower} training-time overhead.

Under the Momentum setting, SMoE can be viewed through a multi-objective optimization lens, where the output norm of each SMoE block generally decreases across layers, with a slight increase in the final layer due to gradient-descent overshooting \cite{teo2024_momentumsmoe}. As shown in \cref{fig:norm_change}, SSR preserves this characteristic norm evolution and closely matches the layer-wise trajectory of Vanilla SMoE, indicating compatibility with Momentum dynamics. Notably, SSR improves Vanilla SMoE by \textcolor{green!70!black}{0.489 PPL}, while prior balancing methods degrade performance by up to \textcolor{red!70!black}{0.778 PPL}. Together with the norm analysis in \cref{fig:norm_change}, this suggests that existing balancing strategies may disrupt Momentum-induced optimization dynamics, leading to less stable trajectories, whereas SSR preserves the norm behavior of Vanilla SMoE while improving perplexity. Overall, SSR improves performance in both settings and trains faster than competing balancing mechanisms.
\begin{figure}[h]
    \centering
\includegraphics[width=0.85\linewidth]{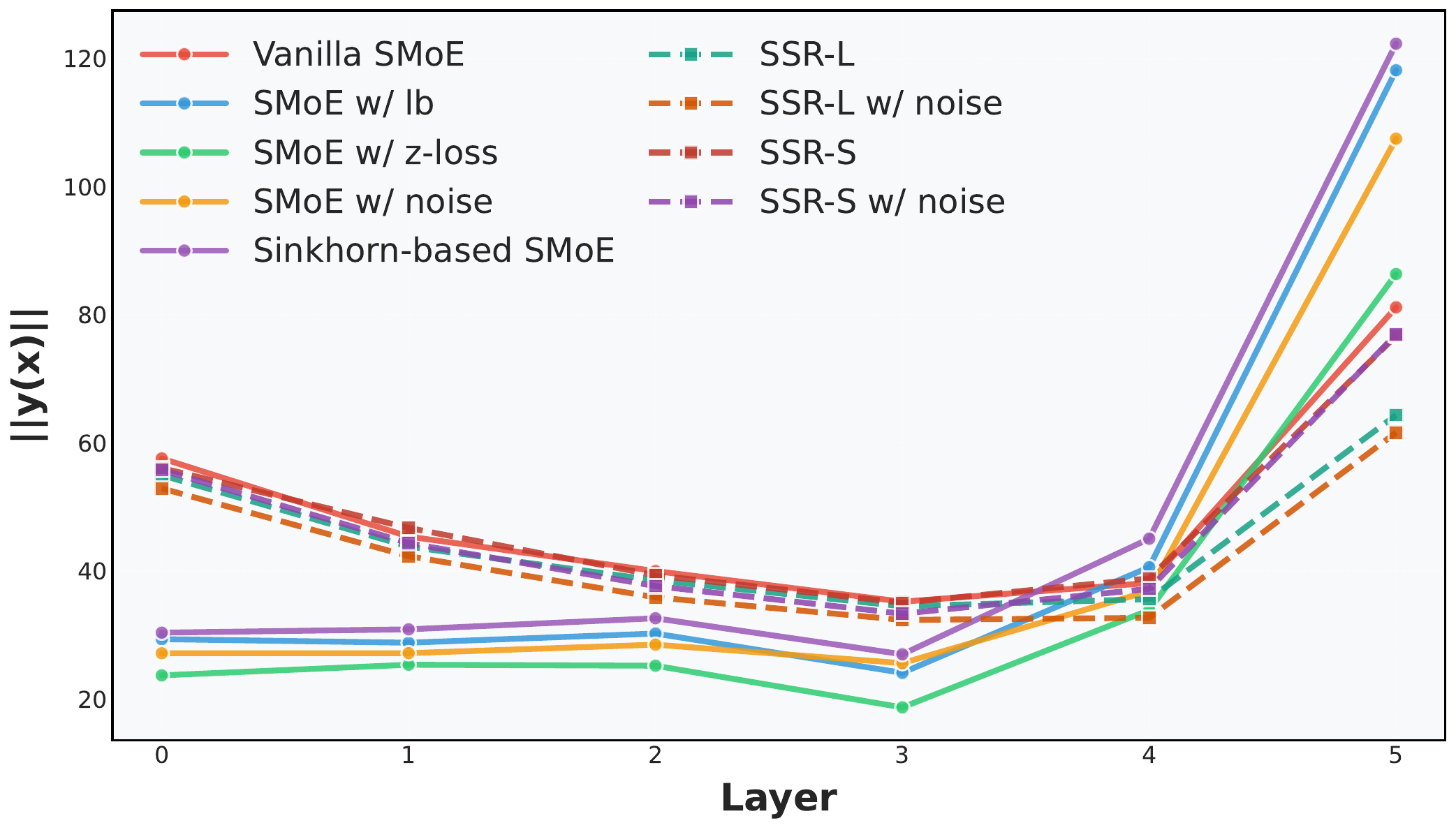}
    \caption{Average output norms across layers on the WikiText-103 validation set under the momentum setting.}
    \label{fig:norm_change}
\end{figure}
\vspace{-1.3em}

\cref{tab:wkt103_inference} compares four inference strategies: inference without balancing (\emph{No balancing}), inference with Sinkhorn-based routing using a fixed probability $p$ (\emph{w/ Sinkhorn}), inference with Gaussian noise (\emph{w/ noise}), and inference with both Sinkhorn routing and noise (\emph{w/ both}). We find that, at inference time, disabling load balancing yields the best results: both SSR variants consistently outperform the other inference strategies, consistent with \cref{prop:sinkhorn_for_training}.

\begin{table}[h]
    \centering
{
    \caption{Comparison of different Inference techniques of SSR-L/S w/ noise on clean/attacked WikiText-103.}
    \label{tab:wkt103_inference}
\renewcommand\tabcolsep{4pt}
\resizebox{0.8\columnwidth}{!}{%
    \begin{tabular}{lcccc} 
    \toprule
      \multirow{2}{*}{Inference} & \multicolumn{2}{c}{Clean WikiText-103} & \multicolumn{2}{c}{Attacked WikiText-103} \\
      \cmidrule(lr){2-3}\cmidrule(lr){4-5}
      & Valid PPL $\downarrow$ & Test PPL $\downarrow$ & Valid PPL $\downarrow$ & Test PPL $\downarrow$ \\
    \midrule
        \multicolumn{5}{c}{SSR-L w/ noise} \\
    \midrule
    No balancing (Ours) & \textbf{31.871} & \textbf{33.395} & \textbf{40.083} & \textbf{41.885} \\
    \midrule
    w/ Sinkhorn & \underline{31.883} & 33.411 & \underline{40.095} & 41.905 \\ 
    w/ noise & \underline{31.883} & \underline{33.407} & 40.099 & \underline{41.904} \\
    w/ both & 31.894 & 33.424 & 40.111 & 41.926 \\
\midrule
    \multicolumn{5}{c}{SSR-S w/ noise} \\
    \midrule
    No balancing (Ours) & \textbf{32.371} & \textbf{33.444} & \textbf{40.842} & \textbf{42.312} \\
    \midrule
    w/ Sinkhorn & \underline{32.377} & \underline{33.447} & \underline{40.847} & \underline{42.315} \\ 
    w/ noise & 32.392 & 33.452 & 40.870 & 42.317 \\
    w/ both & 32.407 & 33.461 & 40.886 & 42.325 \\
    \bottomrule
    \end{tabular}
}
}
\end{table}

\begin{figure}[t]
    \centering

    \includegraphics[width=1\linewidth]{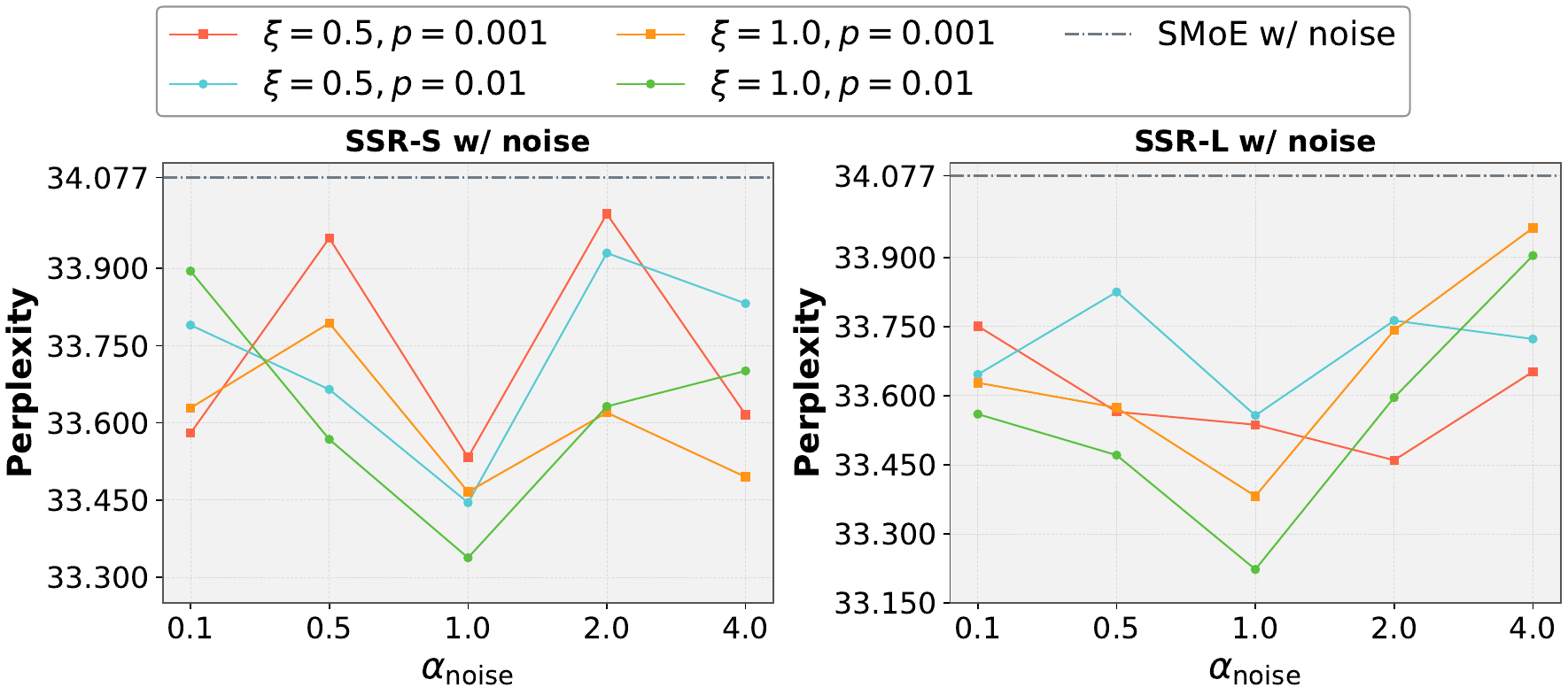}
    \caption{Performance of SSR w/ noise on Wikitext-103 under varying $\xi$ and $\alpha_{\text{noise}}$ values in comparison with SMoE w/ noise.}
    \vspace{-1em}
    \label{fig:ssr_alpha_noise}
\end{figure}

Next, we evaluate SSR on byte-level language modeling with enwik8 (\cref{tab:enwik8-main}). Particularly, SSR variants consistently outperform the baselines. SSR-S w/ noise achieves the best test BPC of 1.128, improving Vanilla SMoE by 0.010 BPC, which is twice the gain of the strongest prior baseline, Sinkhorn-based SMoE. It is also \textbf{2.28$\times$} faster than Sinkhorn-based SMoE, improving both performance and training efficiency (\cref{fig:training_time_enwik8}). SSR-L w/ noise obtains the second-best result, further confirming the effectiveness of SSR for byte-level modeling. In contrast, SMoE w/ noise underperforms Vanilla SMoE, highlighting the limited robustness of existing balancing methods when considering both token- and byte-level language modeling.

Finally, we evaluate SSR in the vision domain (See \cref{table:vision_main}). Due to the high computational cost of Sinkhorn-based methods, we compare against Vanilla SMoE, SMoE with noise, and Swin-MoE. SSR-L achieves strong results on ImageNet-1K and the best performance on ImageNet-O. Adding noise further improves performance across the remaining metrics, underscoring the benefit of noise in the cost matrix.

\begin{table}[ht]
\vspace{-1em}
    \centering
{
    \caption{Bits-per-character (BPC) of SSR-L/S w/ noise vs. SMoE-based \& Sinkhorn-based baselines on the Enwik-8 datasets.}
    \label{tab:enwik8-main}
\renewcommand\tabcolsep{4pt}
\resizebox{0.8\columnwidth}{!}{%
    \begin{tabular}{lccc} 
    \toprule
      \multirow{2}{*}{Model/Metric} & \multicolumn{2}{c}{Enwik-8} & \multirow{2}{*}{\shortstack{$\Delta$ Test BPC vs. \\ vanilla SMoE}}\\
      \cmidrule(lr){2-3}
      & Valid BPC $\downarrow$ & Test BPC $\downarrow$  \\
      \midrule
      {Vanilla SMoE} & 1.147 & 1.138 & -- \\
      {SMoE w/ noise } & 1.154 & 1.147 & \textcolor{red!70!black}{+0.009}\\
      {SMoE w/ aux. loss } & 1.145 & 1.136 & \textcolor{green!70!black}{-0.002} \\
      {Sinkhorn-based SMoE } & 1.143 & \underline{1.133} & \textcolor{green!70!black}{-0.005} \\
      \rowcolor{gray!20}
      {SSR-L w/ noise (Ours)} & \underline{1.141} & \underline{1.133} & \textcolor{green!70!black}{-0.005} \\
      \rowcolor{gray!20}
      {SSR-S w/ noise (Ours)} & \textbf{1.136} & \textbf{1.128} & \textcolor{green!70!black}{-0.010}\\
      \midrule
    \end{tabular}
    }
}

\end{table}

%% file: Section/Ablation.tex
\vspace{-1em}
\subsection{Ablation Studies}
\label{sec:ablation}
\subsubsection{Effect of probability $p$}

We evaluate SSR-L w/ noise under different routing probabilities $p$ to study their effect on SMoE performance in \cref{tab:wkt103_p_w_noise}. As shown in the table, $p$ has a clear impact on performance. On the clean WikiText-103 split, $p=0.001$ achieves the best validation and test PPL, suggesting that a very small fraction of Sinkhorn-based routing is sufficient to provide an effective balancing signal without the conventional auxiliary loss, while on the attacked split, the best performance is obtained with $p=0.0001$. Overall, SSR-L w/ noise benefits from a small but carefully chosen routing probability: overly large $p$ may over-constrain routing, while overly small $p$ weakens the balancing effect.



\begin{table}[ht]
\vspace{-0.1in}
    \centering
{\small
    \caption{Perplexity (PPL) of 1-Head SSR-L w/ noise with different probability $p$ on clean/attacked WikiText-103.}
    \label{tab:wkt103_p_w_noise}
    \scalebox{0.7}{
    \begin{tabular}{lcccc} 
    \toprule
      \multirow{2}{*}{$p$} 
      & \multicolumn{2}{c}{Clean WikiText-103} 
      & \multicolumn{2}{c}{Attacked WikiText-103} \\
      \cmidrule(lr){2-3} \cmidrule(lr){4-5}
      & Valid PPL $\downarrow$ 
      & Test PPL $\downarrow$ 
      & Valid PPL $\downarrow$ 
      & Test PPL $\downarrow$ \\
    \midrule
      {$0.1$} 
      & 32.874 
      & \underline{34.519} 
      & \underline{41.100} 
      & 43.183 \\

      {$0.03$} 
      & 33.254 
      & 35.194 
      & 41.898 
      & 44.211 \\

      {$0.01$} 
      & 32.942 
      & 34.563 
      & 41.159 
      & \underline{43.143} \\

      {$10^{-3}$} 
      & \textbf{32.767} 
      & \textbf{34.367} 
      & 41.164 
      & 43.159 \\

      {$10^{-4}$} 
      & \underline{32.845} 
      & 34.677 
      & \textbf{40.861} 
      & \textbf{42.905} \\

      {$10^{-4}$} 
      & 33.164 
      & 34.935 
      & 41.569 
      & 43.690 \\
    \bottomrule
    \end{tabular}}
}

\vspace{-0.2in}
\end{table}


\subsubsection{Effect of $\alpha_{\text{noise}}$ in SSR-L/S w/ noise} 
We further study the effect of gating noise in SSR-L/S w/ noise by varying $p \in \{0.001, 0.01\}$, $\xi \in \{0.5, 1\}$, and $\alpha_{\text{noise}} \in \{0.1, 0.5, 1, 2, 4\}$. We compare SSR-L/S w/ noise against SMoE w/ noise, the strongest baseline combining load balancing with Momentum in \cref{tab:wikitext-103-unified}. As shown in \cref{fig:ssr_alpha_noise}, both SSR-L and SSR-S w/ noise consistently outperform SMoE w/ noise across all tested settings, achieving the best test PPLs of 33.223 and 33.338, respectively, compared to 34.007 for SMoE w/ noise. These results demonstrate that SSR remains stable and effective under diverse noise and momentum configurations.







\subsubsection{Effect of $\xi$ in Sinkhorn Algorithm}
Next, we examine the effect of the regularization parameter $\xi$ on the performance of SSR w/ noise variants, with fixed settings of $p = 0.001$ and $\delta = 0.0001$. 
As discussed in \cref{sec:method_token_to_expert_assignment}, employing a Softmax-based cost in regularized OT improves stability compared to a linear cost by bounding the cost matrix, thereby mitigating overflow during the Sinkhorn iterations. As shown in \cref{tab:wikitext-103-different-epsilon}, smaller $\xi$ values (e.g., 0.05, 0.1) cause numerical overflow in SSR-L w/ noise (yielding \texttt{NaN}), whereas SSR-S w/ noise remains stable and continues to achieve competitive performance.

\begin{table}[ht]
    \centering
    {
    \caption{Perplexity (PPL) of 1-Head SSR-L/S w/ noise with different $\xi$ on clean/attacked WikiText-103.}
    \label{tab:wikitext-103-different-epsilon}
\renewcommand\tabcolsep{4pt}
\resizebox{0.6\columnwidth}{!}{%
    \begin{tabular}{lcccc} 
        \toprule
        \multirow{2}{*}{$\xi$} & \multicolumn{2}{c}{Clean WikiText-103} & \multicolumn{2}{c}{Attacked WikiText-103} \\
              \cmidrule(lr){2-3} \cmidrule(lr){4-5}

        & Valid PPL $\downarrow$ & Test PPL $\downarrow$ & Valid PPL $\downarrow$ & Test PPL $\downarrow$ \\
        \midrule
        \multicolumn{5}{c}{SSR-L w/ noise} \\
        \midrule
        0.05 & \texttt{NaN} & \texttt{NaN} & \texttt{NaN} & \texttt{NaN} \\
        0.1 & \texttt{NaN} & \texttt{NaN} & \texttt{NaN} & \texttt{NaN} \\
        0.5 & \underline{32.767} & \textbf{34.367} & 41.164 & 43.159 \\
        1 & 32.867 & 34.783 & 41.556 & 43.708 \\
        \midrule
        \multicolumn{5}{c}{SSR-S w/ noise} \\
        \midrule
        0.05  & 32.927 & 34.605 & 41.138 & 43.022 \\
        0.1 & \textbf{32.647} & \underline{34.442} & \textbf{40.834} & 43.059 \\
        0.5 & 32.881 & 34.557 & \underline{41.037} & \textbf{42.941} \\
        1 & 33.053 & 34.564 & 41.056 & \underline{42.945} \\
        \bottomrule
    \end{tabular}}
    }
\vspace{-1em}    
\end{table}

%% file: Section/Conclusion.tex
\vspace{-0.7em}
\section{Conclusion}

This paper proposes Selective Sinkhorn Routing, a novel method for Sparse Mixture-of-Experts models that improves expert utilization with minimal overhead. We reformulate routing as a regularized optimal transport problem with a constraint on the number of tokens per expert. 
Unlike prior methods, we derive gating scores directly from the transport map, leading to more balanced and effective token-to-expert assignments.
Our theoretical and empirical results show that applying Sinkhorn intermittently and injecting noise into the cost matrix reduces training time and improves performance compared to existing routing methods, while the modifications should be disabled at inference and add no extra cost. These results highlight the practicality and effectiveness of our approach for efficient SMoE design across both training and inference deployment.

%% file: Section/Appendix.tex
\clearpage

\section{Theoretical results}

\label{Appendix:Proof}

\subsection{Proposition~\ref{prop:optimal_transport_routing}}
\begin{proof}
Let \( \boldsymbol{\hat{\Pi}}_{i} = (\boldsymbol{\hat{\Pi}}_{i, 1}, \ldots, \boldsymbol{\hat{\Pi}}_{i, n}) \in \mathbb{R}^n \) be the \( i \)-th row of the solution to the entropy-regularized optimal transport problem. We aim to minimize the Kullback-Leibler (KL) divergence:

$$
\mathrm{KL}(\alpha_i \,\|\, \boldsymbol{\hat{\Pi}}_{i}) = \sum_{j=1}^n \alpha_{i, j} \log \left( \frac{\alpha_{i, j}}{\boldsymbol{\hat{\Pi}}_{i, j}} \right),
$$
where \( \alpha_{i, j} \geq 0 \), \( \sum_{j=1}^n \alpha_{i, j} = 1 \), and \( \text{supp}(\alpha_i) \leq k \). This means we want to select at most \( k \) non-zero entries for \( \alpha_i \) such that the KL divergence is minimized.


Let \( \mathcal{T}_i \subset \{1, \dots, n\} \) be the support of \( \alpha_i \), i.e., the indices where \( \alpha_{i, j} > 0 \). We know that \( |\mathcal{T}_i| \leq k \). For fixed \( \mathcal{T}_i \), the problem can be simplified to the following convex optimization problem:

$$
\min_{\substack{\alpha_j \geq 0,\, j \in \mathcal{T}_i \\ \sum_{j \in \mathcal{T}_i} \alpha_j = 1}} \sum_{j \in \mathcal{T}_i} \alpha_j \log \left( \frac{\alpha_j}{\boldsymbol{\hat{\Pi}}_{i, j}} \right).
$$

The Lagrangian for this problem is:

$$
\mathcal{L}(\alpha, \lambda) = \sum_{j \in \mathcal{T}_i} \alpha_j \log \left( \frac{\alpha_j}{\boldsymbol{\hat{\Pi}}_{i, j}} \right) + \lambda \left( \sum_{j \in \mathcal{T}_i} \alpha_j - 1 \right),
$$
where \( \lambda \) is the Lagrange multiplier. Taking the gradient with respect to \( \alpha_j \) and setting it to zero gives:

$$
\frac{\partial \mathcal{L}}{\partial \alpha_j} = \log \left( \frac{\alpha_j}{\boldsymbol{\hat{\Pi}}_{i, j}} \right) + 1 + \lambda = 0 \quad \Rightarrow \quad \alpha_j = \boldsymbol{\hat{\Pi}}_{i, j} e^{-1 - \lambda}.
$$




The KL divergence is given by:

$$
\mathrm{KL}(\alpha_i \,\|\, \boldsymbol{\hat{\Pi}}_{i}) = \sum_{j \in \mathcal{T}_i} \alpha_j \log \left( \frac{\alpha_j}{\boldsymbol{\hat{\Pi}}_{i, j}} \right).
$$

By substituting \( \alpha_j = \frac{\boldsymbol{\hat{\Pi}}_{i, j}}{\sum_{l \in \mathcal{T}_i} \boldsymbol{\hat{\Pi}}_{i, l}} \), we get:

$$
\mathrm{KL}(\alpha_i \,\|\, \boldsymbol{\hat{\Pi}}_{i}) = \sum_{j \in \mathcal{T}_i} \frac{\boldsymbol{\hat{\Pi}}_{i, j}}{\sum_{l \in \mathcal{T}_i} \boldsymbol{\hat{\Pi}}_{i, l}} \log \left( \frac{\boldsymbol{\hat{\Pi}}_{i, j}}{\sum_{l \in \mathcal{T}_i} \boldsymbol{\hat{\Pi}}_{i, l}} \cdot \frac{1}{\boldsymbol{\hat{\Pi}}_{i, j}} \right).
$$

This simplifies to:

$$
\mathrm{KL}(\alpha_i \,\|\, \boldsymbol{\hat{\Pi}}_{i}) = - \log \left( \sum_{j \in \mathcal{T}_i} \boldsymbol{\hat{\Pi}}_{i, j} \right),
$$



To minimize the KL divergence, we need to maximize the sum \( \sum_{j \in \mathcal{T}_i} \boldsymbol{\hat{\Pi}}_{i, j} \). This is achieved by selecting the top-\( k \) largest values of \( \boldsymbol{\hat{\Pi}}_{i, j} \). Therefore, the optimal support set is:

$$
\mathcal{T}_i = \operatorname{TopK}(\boldsymbol{\hat{\Pi}}_{i}, k).
$$

Thus, the optimal \( \alpha_i \) is given by:

$$
\alpha_{i, j} =
\begin{cases}
\dfrac{\boldsymbol{\hat{\Pi}}_{i, j}}{\sum_{l \in \mathcal{T}_i} \boldsymbol{\hat{\Pi}}_{i, l}}, & \text{if } j \in \mathcal{T}_i, \\
0, & \text{otherwise}.
\end{cases}
$$

This completes the proof for \cref{prop:optimal_transport_routing}.

\end{proof}









\subsection{Proposition~\ref{prop:sinkhorn_for_training}}

\begin{proof}
\textbf{(1) Training.} Let $\{x_1, \dots, x_m\} \sim p(x)$ be a batch drawn i.i.d. Let $\pi(x_i) = f(g(x_i)) \in \Delta^n$ be the routing distribution for input $x_i$, and define the batch-average routing vector:

$$
\bar{\pi} = \frac{1}{m} \sum_{i=1}^m \pi(x_i).
$$

By the law of large numbers, as $m \to \infty$, we have:

$$
\bar{\pi} \xrightarrow{\text{a.s.}} \mathbb{E}_{x \sim p(x)}[\pi(x)].
$$

Thus, enforcing constraint \textit{(C3)}:

$$
\boldsymbol{\Pi}^\top \mathbf{1}_m = \tfrac{m}{n} \mathbf{1}_n.
$$

is equivalent to enforcing:

$$
\bar{\pi} = \frac{1}{n} \mathbf{1}_n,
$$
which promotes uniform expert usage across the data distribution. This helps prevent expert underuse and encourages specialization during training.

\medskip

\textbf{(2) Inference.} At inference time, only a single input $x$ is available. Enforcing \textit{(C3)} with $m = 1$ yields:

$$
\pi(x) = \frac{1}{n} \mathbf{1}_n,
$$
which forces uniform routing, ignoring the input-specific score $\mathbf{S}(x)$. This contradicts the goal of expert specialization and leads to suboptimal predictions. Furthermore, a single sample cannot approximate $p(x)$, so the population-level balancing constraints become meaningless.

\medskip

\textbf{Conclusion.} Training with Sinkhorn routing over batches supports expert balancing over the data distribution. In contrast, inference-time routing should be purely based on the input without enforcing expert balance constraints.
\end{proof}




\subsection{Proposition~\ref{prop:noise_positive_selection}}

\begin{proof}
For any fixed $i$, we observe that:
$$
\tilde{g}_i - \tilde{g}_j = (g_i - g_j) + \alpha_{\text{noise}}(\epsilon_i - \epsilon_j).
$$
Since $\epsilon_i, \epsilon_j \sim \mathcal{N}(0, \sigma^2)$; we have $\epsilon_i - \epsilon_j \sim \mathcal{N}(0, 2\sigma^2)$.
Thus:
\begin{align}
\mathbb{P}(\tilde{g}_i > \tilde{g}_j) 
&= \mathbb{P}(\alpha_{\text{noise}}\epsilon_i - \alpha_{\text{noise}}\epsilon_j > g_j - g_i) \\
&= \mathbb{P}\left(\epsilon_i - \epsilon_j > \frac{g_j - g_i}{\alpha_{\text{noise}}}\right) \\
&= \Phi\left( \frac{g_i - g_j}{\sqrt{2} \sigma \alpha_{\text{noise}}} \right).
\end{align}
Independence of the noise implies:
$$
P_i = \prod_{j \ne i} \mathbb{P}(\tilde{g}_i > \tilde{g}_j) = \prod_{j \ne i} \Phi\left( \frac{g_i - g_j}{\sqrt{2} \sigma \alpha_{\text{noise}}} \right).
$$
Since each $g_i - g_j$ is fixed and finite, and \( \Phi(z) \in (0, 1) \) for all real $z$, it follows that:
$$
P_i > 0.
$$
Hence, $P_i$ is a positive constant that depends only on $g$ and $\sigma$, and is independent of any input-dependent softmax transformation. Moreover, in the SSR-S variant, the probability $P_i \geq \Phi^n\left( \frac{-\sqrt{2}}{\sigma \alpha_{\text{noise}}} \right)$, which depends solely on the noise hyperparameter.
\end{proof}

\section{Experimental Details}
\subsection{Dataset Details}
\paragraph{WikiText-103} \cite{merity2017_wikitext} is a large-scale collection comprising over 100 million tokens sourced from 23,805 $``$Good$"$
articles and 4,790 $``$Featured articles$"$.
Specifically, the training set includes over 103 million tokens, while the validation and test sets consist of 217,646 and 245,569 tokens, respectively. 
This dataset retains its original case, punctuation, and numerical values, resulting in a diverse vocabulary of 267,735 unique tokens. A clean version corresponds to the original dataset, while the attacked version is generated using TextAttack’s word-swap attack \cite{morris2020_textattack}, where words in the validation and test sets are randomly replaced with the generic token $``$AAA$"$. This modification increases the difficulty for the model to predict the next word in the sequence accurately.
\paragraph{Enwik-8} \cite{mahoney2011large} is a byte-level dataset comprising 100 million bytes sourced from Wikipedia. It includes not only English text but also markup, special characters, and multilingual content. The dataset is divided into 90 million bytes for training, 5 million for validation, and 5 million for testing.
\paragraph{ImageNet-1K} contains 1.28 million training images and 50,000 validation images. The model is trained to classify each input image into one of 1,000 categories. Top-1 and Top-5 accuracy are reported across all experiments.
\paragraph{ImageNet-A} \cite{hendrycks2021natural} contains real-world images that have been adversarially filtered to mislead existing ImageNet classifiers. A subset of 200 classes is selected from the original 1,000 ImageNet-1K categories, focusing on those where misclassifications would be particularly severe. These 200 classes broadly represent the diversity of categories found in ImageNet-1K.
\paragraph{ImageNet-O} \cite{hendrycks2021natural} contains adversarially filtered examples designed to challenge ImageNet out-of-distribution detectors. It consists of samples drawn from ImageNet-22K that are not part of ImageNet-1K, specifically selected because a ResNet-50 model incorrectly classifies them as ImageNet-1K categories with high confidence.
\paragraph{Imagenet-R} \cite{hendrycks2021many} contains a variety of artistic renditions of object classes originally found in ImageNet, which are typically discouraged by the standard ImageNet guidelines. ImageNet-R includes 30,000 such renditions spanning 200 classes, selected as a subset of the ImageNet-1K categories.


\subsection{Model Architecture and Training Configurations}
For language modeling, we adopt a medium-scale configuration with 6 layers for Wikitext-103 and 8 layers for Enwik-8. Each layer consists of a multi-head self-attention (MHA) block followed by a SMoE block, both with residual connections. Training is performed with a batch size of 48 for 80,000 steps, using a learning rate of 0.0007 with 4,000 warm-up steps and a dropout rate of 0.1. The model uses 8 attention heads for each MHA block and processes sequences of 512 tokens in each batch, with attention spans of 1,024 for Wikitext-103 and 2,048 for Enwik-8. The SMoE module has 16 experts with top-2 routing. The hidden and expert dimensions are 352
. The resulting model sizes are 216M parameters for WikiText‑103 and 36M for Enwik8, aligning with commonly explored scales in recent work \cite{teo2024_momentumsmoe}.

The baseline settings are customized fairly to the paper report. Specifically, for the SMoE w/ noise, we follow the noise initialization in \cite{shazeer2017_outrageous}. The auxiliary loss coefficient is set to 0.01 \cite{shazeer2017_outrageous, lepikhin2020_gshard}, and the $z$-loss coefficient to 0.001 \cite{zoph2022_stmoe}. For the Sinkhorn-based SMoE, we set $\xi =1, \delta =0.0001$ and $\eta = 100$. With SSR-variants, we consider $\xi \in \{0.05, 0.5, 1\}, \delta= 0.0001$, $\eta = 100$, $p \in \{0.0001, 0.001, 0.01\}$ and $\alpha_{\text{noise}} \in \{0.3, 1, 4\}$. 

For image classification, we use a compact 4-stage architecture with depths [2, 2, 18, 2]. The first two stages each have 2 blocks (self-attention + feed-forward); the third stage has 18 blocks, where self-attention alternates between feed-forward and MoE layers; and the final stage includes a self-attention–feed-forward block followed by a self-attention–MoE block. The embedding dimension is 96 with attention heads [3, 6, 12, 24]. We employ 32 experts for MoE layers with top-2 routing (550M parameters) and train for 60 epochs using AdamW (base LR 1.25e-4, min LR 1.25e-7, weight decay 0.1, cosine schedule), batch size 96, and an auxiliary loss coefficient of 0.1. 
For SSR-L and SSR-L w/ noise, we consider $\xi=0.5$, $\delta= 0.0001$, $\eta=100$, $p=0.01$ and $\alpha_{\text{noise}} = 1$. While prior work \cite{teo2024_momentumsmoe} has explored model sizes ranging from 36M to 388M parameters, in this paper we extend the scale further to 550M parameters, demonstrating that SSR continues to provide strong performance gains at larger scale.

\subsection{Compute Resources} All models are trained and evaluated using 2 NVIDIA A100 GPUs with 40GB of memory each. 




\section{Additional Experimental Results}

\begin{table}[ht]
    \centering
{\small
    \caption{Perplexity (PPL) of SSR-L with different probability $p$ on clean/attacked WikiText-103.}
    \label{tab:wkt103_p}
    \scalebox{0.8}{
        \begin{tabular}{lcccc} 
        \toprule
          \multirow{2}{*}{$p$} 
          & \multicolumn{2}{c}{Clean WikiText-103} 
          & \multicolumn{2}{c}{Attacked WikiText-103} \\
          \cmidrule(lr){2-3} \cmidrule(lr){4-5}
          & Valid PPL $\downarrow$ 
          & Test PPL $\downarrow$ 
          & Valid PPL $\downarrow$ 
          & Test PPL $\downarrow$ \\
        \midrule
          {$0.5$} 
          & 33.321 
          & 34.752 
          & 41.748 
          & 43.587 \\

          {$0.1$} 
          & \textbf{32.754} 
          & 34.815 
          & \textbf{40.984} 
          & \underline{43.386} \\

          {$0.03$} 
          & 32.859 
          & 34.712 
          & 41.281 
          & 43.489 \\

          {$0.01$} 
          & 33.051 
          & 34.913 
          & \underline{41.209} 
          & 43.423 \\

          {$10^{-3}$} 
          & 33.029 
          & \underline{34.710} 
          & 41.403 
          & 43.437 \\

          {$10^{-4}$} 
          & \underline{32.809} 
          & \textbf{34.573} 
          & 41.230 
          & \textbf{43.318} \\

          {$10^{-5}$} 
          & 32.992 
          & 34.794 
          & 41.573 
          & 43.821 \\
        \bottomrule
        \end{tabular}
    }
}
\end{table}

\cref{tab:wkt103_p} reports the effect of varying the probability $p$ of applying Sinkhorn routing in SSR-L. Across all tested values, SSR-L achieves lower Test PPL than Vanilla SMoE, showing that sparse Sinkhorn routing provides a useful balancing signal. The best Test PPL is obtained at $p=0.0001$ on both clean and attacked WikiText-103, suggesting that only a very small fraction of Sinkhorn-based routing is needed to improve performance. However, SSR-L still underperforms SSR-L w/ noise in \cref{tab:wikitext-103-unified}, indicating that gating noise further improves the effectiveness of SSR. Overall, these results show that $p$ should remain small: using Sinkhorn routing too frequently may over-constrain the router, while using it too rarely weakens the balancing effect.

\begin{table}[h]
    \centering
{\small
    \caption{Perplexity (PPL) of SSR-S with different probability $p$ on clean/attacked WikiText-103.}
    \label{tab:wk103_ssr-S_p}
    \scalebox{0.8}{
        \begin{tabular}{lcccc} 
        \toprule
          \multirow{2}{*}{$p$} 
          & \multicolumn{2}{c}{Clean WikiText-103} 
          & \multicolumn{2}{c}{Attacked WikiText-103} \\
          \cmidrule(lr){2-3} \cmidrule(lr){4-5}
          & Valid PPL $\downarrow$ 
          & Test PPL $\downarrow$ 
          & Valid PPL $\downarrow$ 
          & Test PPL $\downarrow$ \\
        \midrule
          {$0.1$} 
          & 33.112 
          & 35.112 
          & 41.483 
          & \underline{43.728} \\

          {$0.03$} 
          & 33.033 
          & 35.002 
          & \underline{41.434} 
          & 43.802 \\

          {$0.01$} 
          & 33.249 
          & 35.060 
          & 42.014 
          & 44.187 \\

          {$10^{-3}$} 
          & \textbf{32.610} 
          & \textbf{34.744} 
          & \textbf{41.012} 
          & \textbf{43.283} \\

          {$10^{-4}$} 
          & 33.533 
          & 35.240 
          & 42.036 
          & 44.108 \\

          {$10^{-5}$} 
          & \underline{32.905} 
          & \underline{34.804} 
          & 41.523 
          & 43.895 \\
        \bottomrule
        \end{tabular}
    }
}
\end{table}

\begin{table}[h]
    \centering
{\small
    \caption{Perplexity (PPL) of SSR-S w/ noise with different probability $p$ on clean/attacked WikiText-103.}
    \label{tab:wk103_ssr-S_w_noise_p}
    \scalebox{0.8}{
        \begin{tabular}{lcccc} 
        \toprule
          \multirow{2}{*}{$p$} 
          & \multicolumn{2}{c}{Clean WikiText-103} 
          & \multicolumn{2}{c}{Attacked WikiText-103} \\
          \cmidrule(lr){2-3} \cmidrule(lr){4-5}
          & Valid PPL $\downarrow$ 
          & Test PPL $\downarrow$ 
          & Valid PPL $\downarrow$ 
          & Test PPL $\downarrow$ \\
        \midrule
          {$0.1$} 
          & 33.212 
          & \underline{34.768} 
          & 41.568 
          & \underline{43.411} \\

          {$0.03$} 
          & 33.470 
          & 35.143 
          & 41.960 
          & 43.962 \\

          {$0.01$} 
          & 33.193 
          & 34.918 
          & 41.548 
          & 43.764 \\

          {$10^{-3}$}  
          & \textbf{32.881}
          & \textbf{34.557}
          & \textbf{41.037}
          & \textbf{42.941 }\\

          {$10^{-4}$} 
          & \underline{33.160} 
          & 34.908 
          & \underline{41.476} 
          & 43.498 \\

          {$10^{-5}$} 
          & 33.326 
          & 34.809 
          & 41.804 
          & 43.514 \\
        \bottomrule
        \end{tabular}
    }
}
\end{table}

Similarly, we evaluate SSR-S and SSR-S w/ noise across different probabilities $p$, as reported in \cref{tab:wk103_ssr-S_p} and \cref{tab:wk103_ssr-S_w_noise_p}. For both variants, the best performance is consistently achieved at $p=10^{-3}$ across clean and attacked WikiText-103. This suggests that a small amount of Sinkhorn routing is sufficient to provide an effective balancing signal for SSR-S.

Compared with SSR-S, adding gating noise further improves performance, reducing the clean Test PPL from 34.744 to 34.557 and the attacked Test PPL from 43.283 to 42.941 at the same probability $p=10^{-3}$. These results indicate that gating noise complements sparse Sinkhorn routing by improving routing robustness. Overall, $p$ should be carefully controlled: overly large values may over-constrain routing, while overly small values may weaken the balancing effect.

\begin{figure*}[t]
    \centering

    \begin{minipage}{0.4\linewidth}
        \centering
        \includegraphics[width=\linewidth]{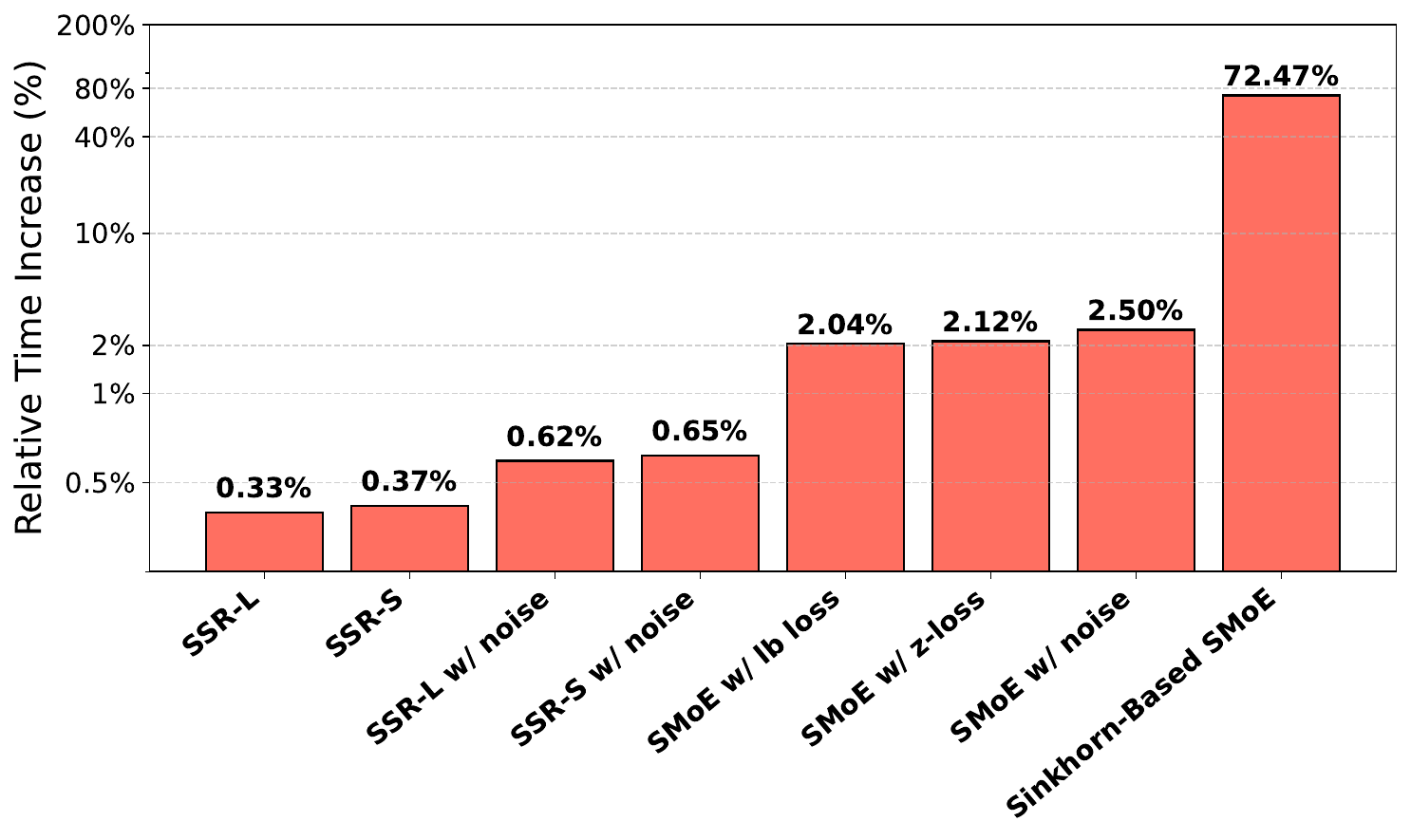}
        \vspace{-0.5em}
        
        \small{(a) Conventional setting}
        \label{fig:wiki_time_overhead_conventional}
    \end{minipage}
    \hspace{2em}
    \begin{minipage}{0.4\linewidth}
        \centering
        \includegraphics[width=\linewidth]{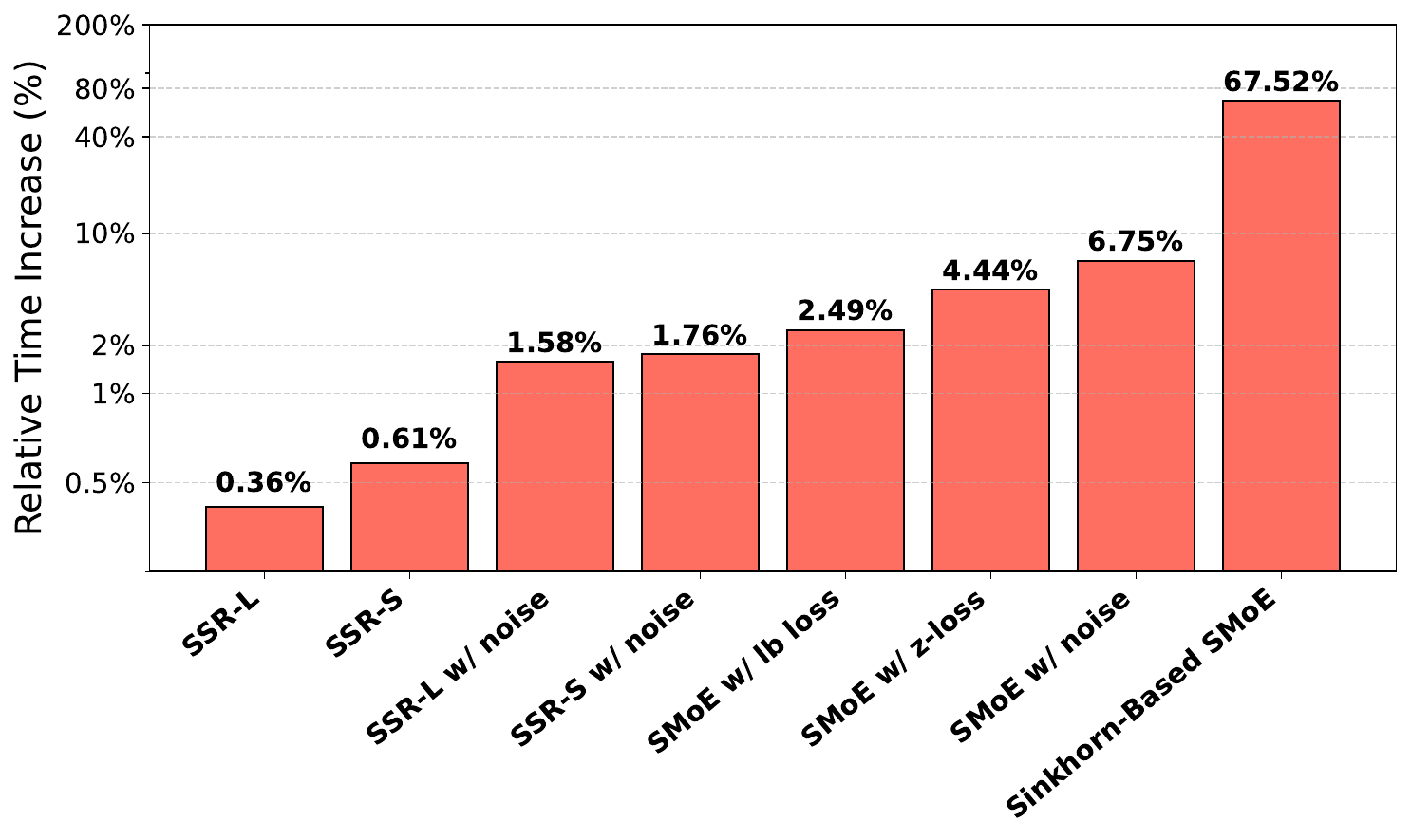}
        \vspace{-0.5em}
        
        \small{(b) Momentum setting}
        \label{fig:wiki_time_overhead_momentum}
    \end{minipage}

    \caption{Training time overhead relative to SMoE (Vanilla) on WikiText-103 under conventional and momentum settings.}
    \label{fig:wiki_training_time_overhead}
\end{figure*}

\begin{figure}[ht]
    \centering
    \includegraphics[width=\linewidth]{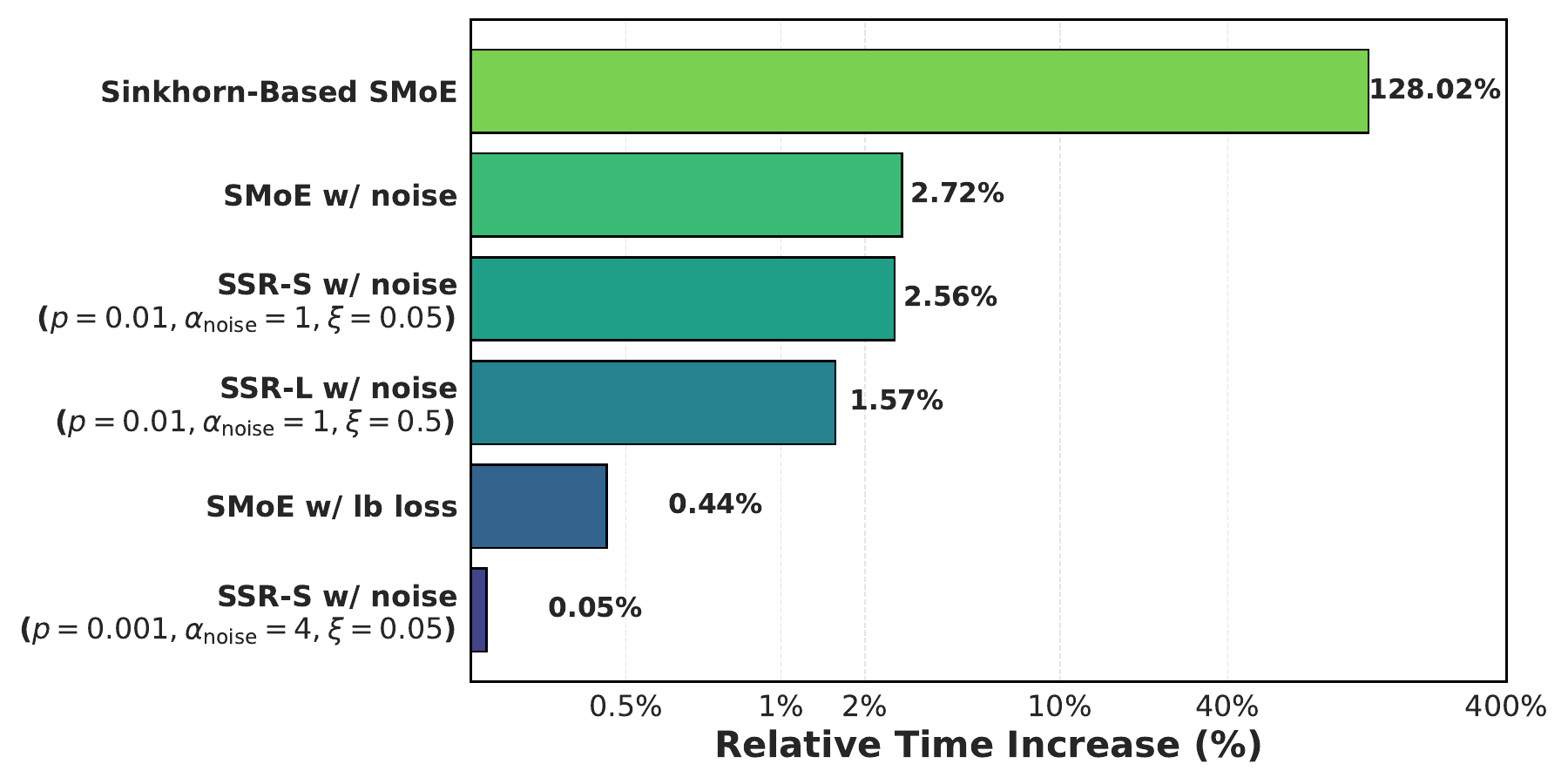}
    \caption{Training time Overhead vs. Vanilla SMoE on Enwik-8 dataset.}
    \label{fig:training_time_enwik8}
\end{figure}

\section{Detailed Derivation: Entropy-Regularized Maximum-Cost OT Solution}
    Recall the entropy-regularized maximum-cost OT problem in Eq.~\ref{OT_lb}:
\begin{align*}
& \boldsymbol{\hat{\Pi}}:=\underset{\boldsymbol{\Pi}}{\arg \max }[\langle\boldsymbol{\Pi}, \mathbf{C}\rangle- \xi\langle\boldsymbol{\Pi}, \log \boldsymbol{\Pi}\rangle], \\
& \text {s.t. }\left\{\begin{array}{l}
\boldsymbol{\Pi}>0, \\
\boldsymbol{\Pi} \mathbf{1}_n=\mathbf{1}_m, \\
\boldsymbol{\Pi}^{\top} \mathbf{1}_m=(m / n) \mathbf{1}_n,
\end{array}\right.
\end{align*}
where $\boldsymbol{\Pi}, \mathbf{C} \in \R^{m \times n}$, $\xi > 0$ and $\langle \mathbf{A}, \mathbf{B} \rangle \coloneqq \displaystyle\sum_{i = 1}^m \displaystyle\sum_{j=1}^n \mathbf{A}_{i,j}\mathbf{B}_{i,j}$. 
Let us consider the Lagrange function:
\begin{align*}
    & \mathbb{L}\left(\boldsymbol{\Pi}, \lambda, \mu\right)  = \displaystyle\sum_{i=1}^m\displaystyle\sum_{j=1}^n \boldsymbol{\Pi}_{i,j}\mathbf{C}_{i,j}- \xi\displaystyle\sum_{i=1}^m\displaystyle\sum_{j=1}^n \boldsymbol{\Pi}_{i,j}\log(\boldsymbol{\Pi}_{i,j}) \\ & + \displaystyle\sum_{i=1}^m \lambda_{i} \left(\displaystyle\sum_{j=1}^n \boldsymbol{\Pi}_{i,j} - 1 \right) + \displaystyle\sum_{j=1}^n \mu_j \left(\displaystyle\sum_{i=1}^m \boldsymbol{\Pi}_{i,j} - (m/n) \right),
\end{align*}
where $\boldsymbol{\Pi} = [\boldsymbol{\Pi}_{i,j}]_{i,j} \in \R^{m \times n}$, $\lambda = [\lambda_{1}, \ldots, \lambda_m], \mu = [\mu_1, \ldots, \mu_n]$ such that $\lambda_i, \mu_j \ge 0$ and $\displaystyle\sum_{i=1}^m  \lambda_{i} = 1, \displaystyle\sum_{j=1}^n \mu_j =1$ for $i = \overline{1, m}, j = \overline{1,n}$. Then we have some following assertions:
\begin{align*}
    &\dfrac{\partial \mathbb{L}\left(\boldsymbol{\Pi}, \lambda\right)}{\partial \boldsymbol{\Pi}_{i, j}} = \mathbf{C}_{i,j} - \xi\log\boldsymbol{\Pi}_{i,j} - \xi + \lambda_{i} + \mu_j  = 0, \\
    &\lambda_i\left(\displaystyle\sum_{j=1}^n \boldsymbol{\Pi}_{i,j} -1\right) = 0, \text{for } i = \overline{1, m},\\
    &\mu_j\left(\displaystyle\sum_{i=1}^m \boldsymbol{\Pi}_{i,j} -m/n\right) = 0, \text{for } j = \overline{1, n}. 
\end{align*}
Hence, we have
\begin{align*}
    \boldsymbol{\Pi}_{i,j} &= e^{(\mathbf{C}_{i,j} - \epsilon + \lambda_i + \mu_j)/\epsilon}.
\end{align*}
Let denote kernel matrix $\mathbf{K} = \exp(\mathbf{C}/\xi - 1_{m \times n})$, then we obtain the solution of the problem in Eq. ~\ref{OT_lb}:
\begin{align*}
    \boldsymbol{\hat{\Pi}} &= \begin{bmatrix}
        e^{\lambda_1 / \xi} & 0 & \ldots & 0 \\
        0 & e^{\lambda_2/ \xi} & \ldots & 0 \\
        \vdots & \vdots & \ddots & \vdots \\
        0 & 0 & \ldots & e^{\lambda_m/ \xi} 
    \end{bmatrix}  \mathbf{K}  \begin{bmatrix}
        e^{\mu_1/\xi} & 0 & \ldots & 0 \\
        0 & e^{\mu_2/\xi} & \ldots & 0 \\
        \vdots & \vdots & \ddots & \vdots \\
        0 & 0 & \ldots & e^{\mu_n/\xi} 
    \end{bmatrix} \\ 
    &= \text{diag}\left(e^{\lambda /\xi}\right) \,\mathbf{K} \,\text{diag}\left(e^{\mu/\xi}\right).
\end{align*}

\textbf{Conclusion.} We obtained the solution of the entropy-regularized maximum-cost OT problem. It is evident that the value range of $\boldsymbol{\hat{\Pi}}$ is directly influenced by the matrix $\mathbf{K}$, which in turn depends on the cost matrix $\mathbf{C}$. Therefore, employing either a linear or softmax cost formulation will inherently affect the solution of Eq.~\ref{OT_lb}.